\documentclass[11pt,a4paper]{article}
\usepackage[T1]{fontenc}
\usepackage[utf8]{inputenc}
\usepackage[a4paper,margin=26mm,headheight=14pt]{geometry}
\usepackage{microtype}
\usepackage{amsmath,amsthm,mathtools}
\usepackage{newtxtext,newtxmath}
\pdfglyphtounicode{parenleftbig}{0028}
\pdfglyphtounicode{parenrightbig}{0029}
\usepackage{graphicx,booktabs,array,tabularx,longtable,multirow}
\usepackage{algorithm}
\usepackage[noend]{algpseudocode}
\usepackage{flafter,placeins}
\usepackage{caption}
\usepackage[round,authoryear]{natbib}
\usepackage{xurl}
\usepackage[hidelinks]{hyperref}
\usepackage{fancyhdr}
\usepackage{needspace,etoolbox}
\AtBeginEnvironment{thebibliography}{\raggedright}
\hypersetup{pdftitle={HAARES: Half-Split Residual Basis Routing for Deep Transformers},
 pdfauthor={Kehan Wang; Tingqiong Cui; Yibu Yang; Yanfei Hao; Chenjie Liu; Shifeng Wu; Zhenzhang Li},
 pdfkeywords={Transformers; residual connections; depth-wise routing; language modeling; Haar-like representations}}

\newcommand{\method}{\ifmmode\mathrm{HAARES}\else HAARES\fi}
\newcommand{\RMS}{\operatorname{RMS}}
\newcommand{\RMSNorm}{\operatorname{RMSNorm}}
\newcommand{\sg}{\operatorname{sg}}
\newcommand{\clip}{\operatorname{clip}}

\newtheorem{proposition}{Proposition}
\BeforeBeginEnvironment{proposition}{\Needspace{7\baselineskip}}
\fancypagestyle{plain}{\fancyhf{}\fancyfoot[C]{\thepage}}

\title{\method{}: Half-Split Residual Basis Routing\\for Deep Transformers}
\author{Kehan Wang$^{1,\dagger}$, Tingqiong Cui$^{3,\dagger}$, Yibu Yang$^{3}$,\\
Yanfei Hao$^{3}$, Chenjie Liu$^{3}$, Shifeng Wu$^{2,*}$, and Zhenzhang Li$^{2,*}$\\[6pt]
\small $^{1}$School of Mechanical Engineering, Chongqing University, Huxi Campus,\\
\small Chongqing 401331, China\\[2pt]
\small $^{2}$School of Computer Science, Guangdong Polytechnic Normal University,\\
\small No. 293 Zhongshan Avenue West, Guangzhou 510665, Guangdong, China\\[2pt]
\small $^{3}$Technology R\&D Department, Chongqing CRRC Group Co., Ltd.,\\
\small No. 99 Shujugu Middle Road, Chongqing 401122, China\\[4pt]
\small $^{\dagger}$These authors contributed equally.\\
\small $^{*}$Correspondence: \href{mailto:fengtree@126.com}{fengtree@126.com} (Shifeng Wu);\\
\small \href{mailto:zhenzhangli@gpnu.edu.cn}{zhenzhangli@gpnu.edu.cn} (Zhenzhang Li)}
\date{}
\begin{document}
\maketitle
\begin{abstract}
Block-level residual routing compresses ordered Transformer sublayer outputs into a single sum, preserving their total but discarding their allocation within the block. We introduce HAARES, which retains the cumulative source and adds the first-half sum minus the second-half sum. The raw pair recovers both half-block totals exactly. A bounded root-mean-square matching operation prepares the detail for the existing depth router, while the final readout remains cumulative-only. In 48-layer OpenWebText experiments, HAARES reduces mean best validation loss by 0.0313 at 201M parameters, 0.0233 at 453M parameters, and 0.0150 with a 4096-entry byte-pair tokenizer. All eight paired runs improve over Block Attention Residuals. Mean losses also decrease on four separately trained 48-layer domain benchmarks. The depth and granularity studies locate the strongest benefit in the four-block, 48-layer configuration: shallower results are mixed, and six- or eight-block partitions reverse the mean advantage. Duplicated cumulative sources recover part of the improvement, but the half-split detail performs better in every ablation seed; the tested random-sign detail does not reproduce the gain. Scale matching and learned detail biases further improve the result. Additional analytical arithmetic is approximately $0.2$--$0.5\%$, while the unfused implementation increases peak memory by $15$--$33\%$ and throughput-derived step time by $17$--$23\%$. HAARES improves validation quality in deep models with long residual blocks, at a measurable memory and runtime cost.
\end{abstract}

\section{Introduction}

A block residual router can decide how much to use an earlier block, but its decision is limited by what that block stores. Block Attention Residuals summarize a group of attention and multilayer-perceptron (MLP) outputs by their sum \citep{kimi2026attnres}. This reduces the number of sources read by later sublayers. It also makes two output sequences indistinguishable whenever they have the same total, even if their early and late contributions differ substantially.

Consider a block that produces a large early contribution and partly cancels it later. Its cumulative source can equal that of a block with a smaller early contribution and little late activity. The sum exposes the endpoint of this allocation, not its two parts. Increasing the expressiveness of the scoring rule cannot retrieve a distinction that its input source has already discarded. This limitation concerns what a block retains: changing how later layers weight its sum cannot recover the missing allocation.

We introduce \method{}, which augments the cumulative source with the first-half output sum minus the second-half sum. The two raw coordinates recover the early and late totals exactly by addition and subtraction. This is the coarsest sum--difference construction associated with a Haar-like representation along depth \citep{mallat1989theory}. It gives the router a structured intra-block contrast while keeping the source count proportional to the number of blocks.

The choice of one fixed split keeps the intervention small. Both coordinates are accumulated online from the same sublayer outputs. The cumulative source is unchanged; the detail is magnitude-matched to it within a bounded range and receives a learned routing bias. Intermediate attention and MLP sublayers read the completed pairs and the current written prefix. The final readout still uses only cumulative sources. Figure~\ref{fig:architecture} follows this construction through source preparation, routing, and write-back.

This design occupies a different point from retaining every sublayer output. With four blocks and 48 Transformer layers, one cumulative source summarizes 24 residual events. HAARES adds one early--late contrast to that block instead of exposing all 24 outputs individually. The raw pair retains the two half totals but still compresses variation within each half. The router gains access to the two half totals while using a much smaller bank than per-sublayer storage would require.

The experiments identify where that choice is useful. Under a matched training recipe, the medium-width OpenWebText model improves by 0.0078 in mean validation loss at 12 layers, is slightly worse by 0.0016 at 24 layers, and improves by 0.0313 at 48 layers. At 48 layers, every paired seed favors HAARES in the 201M, 453M, and BPE comparisons. Four additional corpora, each trained separately, also favor HAARES at 48 layers. The effect is strongest and most consistent when each of four blocks compresses a long residual sequence.

The granularity and source controls make this result more informative than a single model comparison. At fixed 48-layer depth, increasing the block count to six or eight reverses the mean HAARES advantage over the same-partition parent. Duplicating cumulative sources improves the parent, but full HAARES improves further in all three ablation seeds. A fixed random-sign detail does not reproduce the result. These controls distinguish the benefit of another weighted slot from the benefit of the chosen early--late source.

The raw representation and the routed computation require separate analysis. Recoverability of the half totals is exact, but a softmax router acts on normalized keys and nonnegative weights over supplied values. Its behavior depends on the coordinates in which those values are represented. We derive how a detail changes the effective coefficients of the two halves, why equal-valued duplicate slots can affect allocation, and why a half-split block is not operationally equivalent to twice as many cumulative blocks. These properties connect the source definition to the alternatives tested in the ablations.

The computational trade is equally concrete. HAARES adds few parameters and approximately $0.2$--$0.5\%$ analytical arithmetic, yet the unfused implementation incurs a larger activation and runtime cost. We report measured peak memory and throughput alongside validation loss. Existing checkpoint crossings show that earlier attainment of the baseline's budget-best quality can offset the slower steps relative to a full baseline run, while the additional memory remains necessary.

The resulting contribution is a source representation for block residual routing, an exact account of its retained resolution and mixture geometry, and matched evidence across depth, width, tokenization, corpora, and design controls. The successful configurations and the reversals under finer partitions help identify when the added contrast is useful and when cumulative sources alone perform better.

\paragraph{Relationship to the preliminary version.}
A preliminary version appeared as an arXiv preprint \citep{wang2026haarespreprint}. This manuscript develops the raw-summary and routing analysis, clarifies the relationship to other residual architectures, and consolidates the existing seed-level, domain, ablation, and cost results. The empirical results are shared with the preliminary version; the extensions concern their analysis and presentation.

\section{Related work}

\subsection{Residual scaling and cross-layer access}

Residual connections provide additive paths through deep networks \citep{he2016deep}, and their interaction with normalization affects Transformer optimization \citep{vaswani2017attention,xiong2020layer}. ReZero learns zero-initialized scalar residual gates \citep{bachlechner2021rezero}. LayerScale learns channelwise scales initialized to small values \citep{touvron2021going}, while DeepNet changes residual scaling and initialization to train much deeper stacks \citep{wang2022deepnet}. These methods regulate the updates made to a residual stream. HAARES addresses the representation of a group of updates once that group becomes a source for later sublayers.

Cross-layer aggregation changes access to earlier computation. Transparent Attention mixes encoder layers for neural machine translation \citep{bapna2018transparent}. DenseFormer learns depth-weighted averages of current and earlier representations \citep{pagliardini2024denseformer}. MUDDFormer generates token-dependent dense-connection weights for separate query, key, value, and residual streams \citep{xiao2025muddformer}. These methods let later computations select among earlier representations. HAARES retains a block router and changes the summaries it can reuse, rather than introducing multiway streams or a new connection-weight generator.

The storage and weighting questions are related but distinct. Retaining many earlier representations gives a weighting mechanism finer access at greater source cost. Summing several outputs into one block source reduces that cost, but subsequent weights act on the aggregate. HAARES uses two structured coordinates per block to retain more intra-block information while preserving a block-level interface. We use the cumulative-only parent as the direct experimental comparison. Standard Residual, ReZero, and LayerScale provide comparisons with other residual parameterizations under the same training recipe.

\subsection{Attention Residuals and delta routing}

Attention Residuals replace uniform accumulation with softmax-weighted reuse of preceding sublayer outputs. Their block variant forms local sums of outputs within contiguous groups and routes over those sources \citep{kimi2026attnres}. A cumulative block source is the sum within that block, not a hidden state accumulated from the embedding through all preceding layers. HAARES inherits this local block-sum organization, the learned read queries, and the replacement-style input mixture.

Delta Attention Residuals route residual changes at sublayer or block granularity and add the routed contribution to a preserved residual stream \citep{luo2026delta}. HAARES instead pairs a local block sum with its intra-block early--late contrast and retains the replacement-style read. These differences concern both the stored quantity and how it enters computation. Under routed inputs, a sublayer output need not equal the difference between two consecutive routed hidden states. Block-level Delta routing retains these differences in both source definition and residual update.

Four architectural choices organize the comparison: the stored sources, the rule that weights them, the use of the mixture as a replacement or an additive correction, and the source family used by the final readout. HAARES modifies the first choice within Block AttnRes and adds detail-specific scale and bias handling. Its final readout continues to use cumulative sources only. This localization makes the parent comparison informative about the proposed source addition. DenseFormer, MUDDFormer, and Delta Attention Residuals are discussed here as architectural relatives, not included as experimental baselines.

\subsection{The Haar-like sum--difference construction}

A Haar scaling/detail pair separates the sum of two adjacent regions from their difference \citep{mallat1989theory}. HAARES applies that construction to vector-valued sublayer outputs along depth. With equal-length halves, its raw coordinates are proportional to the coarsest Haar coefficients. With unequal halves, the definition retains unweighted sums, assigning the extra event to the early half.

The representation contains one split per block. It is neither a multilevel decomposition nor an orthonormal transform after RMS matching. The relevant property is the exact inverse between two half totals and their raw sum and difference. The router then acts on those supplied directions through normalized keys and nonnegative weights. This separates the role of the Haar-like representation from the distinct question of which mixtures the network can produce.

\section{Half-split residual basis routing}
\label{sec:method}

HAARES follows the four stages in Figure~\ref{fig:architecture}: collect ordered sublayer outputs, construct cumulative and early--late sources, prepare and route those sources, and write the next sublayer output back to the active block. The change is confined to the source bank. Attention, MLP transformations, and the cumulative-only final readout retain the parent architecture's interfaces.

The diagram separates a block's stored representation from the computation that produces it. A completed block supplies a fixed pair to later read sites. An active block supplies only the prefix already written. Each new attention or MLP output extends that prefix, after which the next read can use the updated pair. Thus the panel sequence describes one pass through source construction and use; the write-back arrow advances to the next residual event.

\subsection{Ordered sublayer outputs}

Consider a decoder-only Transformer with $L$ layers, hidden width $d$, batch size $B$, and sequence length $T$. Attention and MLP sublayers form $2L$ ordered residual events. We partition these events into $N$ contiguous blocks, each containing $m=2L/N$ events in the equal-sized configurations used here. A block in this partition can contain several Transformer layers; it is the unit of residual storage, rather than an additional neural transformation.

At event $r$ of block $n$, the depth router supplies $\widetilde h_{n,r}\in\mathbb R^{B\times T\times d}$. The attention or MLP transformation $f_{n,r}$ produces
\begin{equation}
 u_{n,r}=f_{n,r}\!\left(\operatorname{Norm}(\widetilde h_{n,r})\right).
 \label{eq:sublayer}
\end{equation}
The ordered outputs $(u_{n,1},\ldots,u_{n,m})$ are the trajectory summarized by HAARES. Each $u_{n,r}$ is written to the block accumulator. Successive sublayers have different depth routers, so their routed inputs need not differ by $u_{n,r}$; defining the trajectory through the sublayer outputs fixes that ambiguity.

The baseline compresses all $m$ outputs into their sum. At one token position this maps $md$ coordinates to $d$ coordinates. Perturbations $(v_1,\ldots,v_m)$ with $\sum_r v_r=0$ leave the source unchanged. Once such outputs have been summed, a router receiving only the total has no access to their allocation within the block. HAARES retains one additional projection of this same sequence: a positive sum over its early portion and a negative sum over its late portion.

The split is positional. Attention and MLP outputs can appear in either half, and no semantic role is assigned to the halves. For the default 48-layer, four-block model, each block groups 24 residual events, with 12 on each side of the split. This groups six Transformer layers per half while adding only one detail source per block.

\begin{figure}[!tbp]
\centering
\includegraphics[width=119mm]{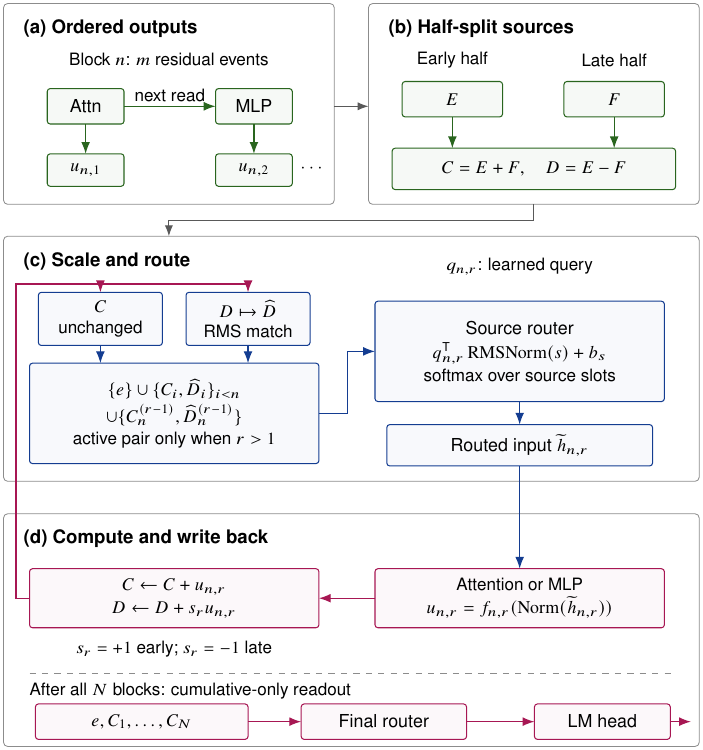}
\caption{HAARES information flow. (a) A block groups ordered attention and MLP outputs. (b) The early and late totals produce $C=E+F$ and $D=E-F$. (c) Only $D$ is RMS-matched; a learned read-site query scores the embedding, completed pairs, and the active written prefix to form $\widetilde h_{n,r}$. (d) The sublayer output updates both raw accumulators, and the return path prepares the next read. The separate end-of-stack router receives only $e,C_1,\ldots,C_N$ before the language-model head}
\label{fig:architecture}
\end{figure}

\subsection{Cumulative and early--late source construction}

Let $k=\lceil m/2\rceil$. Define the early and late totals by
\begin{equation}
 E_n=\sum_{r=1}^{k}u_{n,r},\qquad
 F_n=\sum_{r=k+1}^{m}u_{n,r}.
 \label{eq:segments}
\end{equation}
The stored raw coordinates are
\begin{equation}
 C_n=E_n+F_n,\qquad D_n=E_n-F_n.
 \label{eq:bases}
\end{equation}
Both coordinates depend on both halves: $C_n$ retains their total contribution, while $D_n$ separates their allocation. For odd $m$, the early half contains the extra event; every reported configuration has even $m$. The cumulative coordinate is exactly the one stored by Block AttnRes. Table~\ref{tab:interface} identifies the additions to that interface.

\begin{table}[tbp]
\centering
\caption{The source interface. A ``pair'' refers to two indexed source slots, even when their tensor values coincide. The embedding is available in both methods.}
\label{tab:interface}
\small
\begin{tabularx}{\textwidth}{@{}p{0.27\textwidth}p{0.27\textwidth}X@{}}
\toprule
Quantity & Block AttnRes & HAARES \\
\midrule
Completed block & $C_n=E_n+F_n$ & $C_n$ and $\widehat D_n$, with raw $D_n=E_n-F_n$ \\
Active block before event $r>1$ & $C_n^{(r-1)}$ & $C_n^{(r-1)}$ and $\widehat D_n^{(r-1)}$ \\
Detail magnitude & Not applicable & Bounded RMS matching to $C$; scale detached from gradients \\
Intermediate input & Softmax mixture of sources & Same mixture rule, with a learned bias for detail slots \\
Final readout & Embedding and completed $C_n$ & Unchanged; no detail slots \\
Information in raw summary & Total update & Separate early and late totals; no order within either half \\
\bottomrule
\end{tabularx}
\end{table}

The pair can be maintained as outputs arrive. Before event $r$, the active block contains
\begin{equation}
 C_n^{(r-1)}=\sum_{t<r}u_{n,t},\qquad
 D_n^{(r-1)}=\sum_{t<r}s_tu_{n,t},\qquad
 s_t=\begin{cases}+1,&t\leq k,\\-1,&t>k.\end{cases}
 \label{eq:partial}
\end{equation}
After the next output is computed, the two writes are
\begin{equation}
 C_n^{(r)}=C_n^{(r-1)}+u_{n,r},\qquad
 D_n^{(r)}=D_n^{(r-1)}+s_ru_{n,r}.
 \label{eq:update}
\end{equation}
Starting from zero, induction recovers~\eqref{eq:partial} at each event and~\eqref{eq:bases} at the block boundary. Two running accumulators therefore suffice for the forward summary construction. Training also retains the computation graph required by backpropagation; the measured activation cost is reported separately in Section~\ref{sec:efficiency}.

During the early half, the two raw accumulators coincide. A late output is then added to $C$ and subtracted from $D$. After $k+j$ outputs, write the observed late total as $F_{n,j}=\sum_{t=k+1}^{k+j}u_{n,t}$. The active coordinates are
\begin{equation}
 C_n^{(k+j)}=E_n+F_{n,j},\qquad
 D_n^{(k+j)}=E_n-F_{n,j},\qquad 0\leq j\leq m-k.
 \label{eq:onlinehalves}
\end{equation}
The completed early contribution remains encoded throughout the late half, without a third persistent accumulator. Each additional late output changes the comparison against that same early total. At the block boundary, the pair is cached for later blocks and the active accumulators are reset.

Table~\ref{tab:timing} makes the read timing explicit. The first late event, $r=k+1$, still reads the early-only pair. Its output creates the first early--late contrast, which becomes available at $r=k+2$. This one-event delay follows the read--compute--write order and keeps every read dependent only on preceding outputs. An attention output is available to the following MLP as soon as it has been written.

\begin{table}[tbp]
\centering
\caption{Read timing in the first block, with $m\geq4$ and $k=\lceil m/2\rceil$. Entries show the raw active accumulators before the listed event. Whenever a detail slot is exposed, $D$ is RMS-matched first. The embedding is always available.}
\label{tab:timing}
\small
\begin{tabularx}{\textwidth}{@{}p{.19\textwidth}p{.20\textwidth}p{.20\textwidth}X@{}}
\toprule
Read site & Active $C$ & Active $D$ & What the read can use\\
\midrule
$r=1$ & $0$ & $0$ & No active slots; embedding only\\
$r=2$ & $u_{1,1}$ & $u_{1,1}$ & Two raw values coincide\\
$r=k+1$ & $E_1$ & $E_1$ & Early half complete; no late output yet\\
$r=k+2$ & $E_1+u_{1,k+1}$ & $E_1-u_{1,k+1}$ & First early--late contrast is visible\\
Block 2, $r=1$ & Reset to $0$ & Reset to $0$ & Completed $(C_1,\widehat D_1)$ remains; no new active slots\\
\bottomrule
\end{tabularx}
\end{table}

\subsection{Magnitude matching and depth-wise routing}

A signed difference can have a different magnitude from its corresponding sum. We retain $C$ unchanged and prepare the detail value as
\begin{equation}
 \widehat D=D\cdot\sg\!\left[\clip\!\left(
 \frac{\RMS(C)}{\RMS(D)+\epsilon},\frac1\gamma,\gamma\right)\right],
 \qquad \epsilon=10^{-6},\quad\gamma=4.
 \label{eq:rmsmatch}
\end{equation}
Here $\sg$ detaches the multiplier from the gradient graph. The multiplier stays in $[1/4,4]$, moving the detail toward the cumulative magnitude while limiting amplification and attenuation. Clipping and the denominator offset make this bounded magnitude matching rather than exact RMS equality.

Magnitude matching and router key normalization serve different purposes. The former changes the value supplied to the mixture. The latter constructs the features used to score a source. Normalizing keys alone would leave a large detail value large in the weighted sum. A type bias changes its softmax weight instead of its tensor magnitude. The definition and ablations treat value matching, key normalization, and the detail bias separately.

Before event $(n,r)$, the source bank is the indexed collection
\begin{equation}
\begin{aligned}
 \mathcal S_{n,r}&=\{e\}\cup\{C_i,\widehat D_i\}_{i<n}\cup\mathcal P_{n,r},\\
 \mathcal P_{n,r}&=\begin{cases}
 \varnothing,&r=1,\\
 \{C_n^{(r-1)},\widehat D_n^{(r-1)}\},&r>1.
 \end{cases}
\end{aligned}
\label{eq:sourceset}
\end{equation}
The embedding $e$ is always present. Completed blocks supply their cached pairs, and the active block contributes only after its first output. The braces denote source slots: coincident tensor values remain separate candidates. In particular, the early-only active pair occupies two slots even though its raw coordinates coincide.

For a slot $s$, the router computes
\begin{equation}
 z_s=q^\top\RMSNorm(s)+b_s,\qquad
 \alpha_s=\frac{\exp z_s}{\sum_{s'\in\mathcal S_{n,r}}\exp z_{s'}},\qquad
 \widetilde h_{n,r}=\sum_{s\in\mathcal S_{n,r}}\alpha_s s.
 \label{eq:routingweights}
\end{equation}
Each attention and MLP read site has its own learned query $q\in\mathbb R^d$. It is shared across sequence positions, with token-dependent scores obtained from the source keys. Thus the query is a model parameter, while the depth weights respond to the token representations. The depth softmax mixes source slots at a token position; causal self-attention within $f_{n,r}$ remains responsible for interaction across positions.

Queries start at zero. Embedding and cumulative slots have zero bias. Each sublayer router has one learned scalar detail bias, shared by its visible detail slots and initialized at $-2$. With $c$ embedding/cumulative slots and $h$ detail slots, the initial aggregate detail weight is
\begin{equation}
 p_D^{(0)}=\frac{h\exp(-2)}{c+h\exp(-2)}.
 \label{eq:initialmass}
\end{equation}
This gives a soft initial preference for cumulative sources. For example, the first active read has two embedding/cumulative slots and one detail slot, so the detail group initially receives about $6.3\%$ of the mixture. In the last block of a four-block model, an active read has five embedding/cumulative and four detail slots, giving about $9.8\%$. These values follow from the initialization and slot counts; learned scores and biases subsequently set the allocation.

The clipping rule also specifies the degenerate cases. If $C=0$ and $D\ne0$, it retains $\widehat D=D/4$, preserving a nonzero candidate when the total cancels. If $D=0$, the value stays zero, but its indexed slot still contributes to softmax normalization. For a detached multiplier $a$, the direct derivative of the scaling operation with respect to $D$ is $aI$; there is no derivative through $a$'s dependence on either accumulator. The recorded RMS formula does not identify its reduction axes. That implementation detail is listed with the other reproducibility items in Appendix~\ref{app:repro}.

\subsection{Sublayer update and final readout}

The mixture $\widetilde h_{n,r}$ is the input to the ordinary normalized attention or MLP transformation in~\eqref{eq:sublayer}. Its output $u_{n,r}$ updates the two accumulators through~\eqref{eq:update}. There is no separate operation that adds the mixture to a retained hidden stream. The write-back path in Figure~\ref{fig:architecture} carries only the newly computed sublayer output to the accumulators; it does not introduce a second residual branch.

Algorithm~\ref{alg:haares} states the execution order. One attention or MLP transformation is evaluated at each event. The same output enters $C$ with a positive sign and $D$ with the half-dependent sign. Preparing $\widehat D$ after the write makes the active pair ready for the next read. Completed pairs are reused by downstream routers without recomputing the earlier sublayers.

\begin{algorithm}[!htbp]
\caption{HAARES execution for block $n$}
\label{alg:haares}
\begin{algorithmic}[1]
\Require embedding $e$, completed pairs $\{C_i,\widehat D_i\}_{i<n}$, sublayers $f_{n,1:m}$
\State $C\gets0$; $D\gets0$; $k\gets\lceil m/2\rceil$
\For{$r=1,\ldots,m$}
 \State Assemble $\mathcal S_{n,r}$ using~\eqref{eq:sourceset}; omit the active pair when $r=1$
 \State $\widetilde h\gets\operatorname{Route}_{n,r}(\mathcal S_{n,r})$
 \State $u\gets f_{n,r}(\operatorname{Norm}(\widetilde h))$
 \State $C\gets C+u$
 \State $D\gets D+u$ if $r\leq k$, otherwise $D\gets D-u$
 \State $\widehat D\gets\operatorname{RMSMatch}(D;C)$ using~\eqref{eq:rmsmatch}
\EndFor
\State \Return completed pair $(C,\widehat D)$
\end{algorithmic}
\end{algorithm}

The matched Block AttnRes baseline uses the same read rule with every detail slot removed. After the last block, both architectures use a separate final router over
\begin{equation}
 \mathcal S_{\mathrm{final}}=\{e,C_1,\ldots,C_N\},\qquad
 h_{\mathrm{final}}=\operatorname{Route}_{\mathrm{final}}(\mathcal S_{\mathrm{final}}).
 \label{eq:final}
\end{equation}
The tied language-model head maps $h_{\mathrm{final}}$ to token logits. Details affect the predictions through the intermediate sublayer outputs that build the cumulative sources. They are not candidates in the final readout. This preserves the parent's output interface while placing the extra early--late information at the inputs of the attention and MLP computations.

\FloatBarrier
\section{Information retained by the half split}
\label{sec:analysis}

The half split changes the coordinates available to a depth router. We first characterize the raw summary for fixed realized sublayer outputs, then examine how magnitude matching and softmax weighting use those coordinates. Fixing the outputs is essential here: the analysis concerns compression of a computed trajectory, whereas changing a neural sublayer also changes the outputs computed after it.

\subsection{Recovery of half-block totals}

\begin{proposition}[Recovery of half-block totals]
The raw pair $(C_n,D_n)$ uniquely determines $(E_n,F_n)$ through
\begin{equation}
 E_n=\frac{C_n+D_n}{2},\qquad F_n=\frac{C_n-D_n}{2}.
 \label{eq:inverse}
\end{equation}
\end{proposition}
\begin{proof}
Adding the defining equations $C_n=E_n+F_n$ and $D_n=E_n-F_n$ gives $2E_n$; subtracting them gives $2F_n$.
\end{proof}

The pair therefore preserves the two half totals exactly. A cumulative source alone identifies every pair $(E_n+a,F_n-a)$ with $(E_n,F_n)$, for arbitrary $a\in\mathbb R^d$ at a token position. Adding the detail distinguishes these pairs because the detail changes by $2a$. This is the information gained by the additional source: the total update can stay fixed while its allocation between the two halves changes.

For arbitrary scalar coefficients $a,b$, the same inverse gives
\begin{equation}
 aE_n+bF_n=\frac{a+b}{2}C_n+\frac{a-b}{2}D_n.
 \label{eq:span}
\end{equation}
The raw sum--difference coordinates thus span every linear combination of the two half totals. Two $d$-dimensional vectors are sufficient for arbitrary pairs of $d$-dimensional totals. To see why one linear summary is insufficient, stack $(E,F)$ into a $2d$-dimensional vector. The cumulative map has matrix $[I\ I]$ and kernel $\{(a,-a):a\in\mathbb R^d\}$, of dimension $d$. The forward sum--difference map and its inverse are
\begin{equation}
 \begin{bmatrix}I&I\\ I&-I\end{bmatrix},\qquad
 \frac12\begin{bmatrix}I&I\\ I&-I\end{bmatrix}.
 \label{eq:blockmatrix}
\end{equation}
Their product is the identity on $\mathbb R^{2d}$. Recovery holds even when $C$ and $D$ happen to be collinear for a particular input.

For example, let $a\in\mathbb R^d$ be nonzero and consider
\begin{equation}
 (E,F)=(3a,-2a),\qquad (E',F')=(a,0).
 \label{eq:example}
\end{equation}
Both pairs give $C=C'=a$, while their details are $5a$ and $a$. The first pair contains a larger early contribution that is partly cancelled by the late contribution. The second reaches the same total entirely in the early half. Their cumulative sources are indistinguishable, but their raw pairs retain the difference. The example illustrates which allocations the summary can distinguish. The experiments evaluate trained models, not the frequency of this particular cancellation pattern.

Expanding the squared norms yields
\begin{equation}
 \lVert C_n\rVert_2^2+\lVert D_n\rVert_2^2
 =2\bigl(\lVert E_n\rVert_2^2+\lVert F_n\rVert_2^2\bigr),
 \label{eq:energy}
\end{equation}
with tensor norms interpreted after flattening. The cross terms cancel in the sum. Subtracting the two squared norms instead gives
\begin{equation}
 \lVert D_n\rVert_2^2-\lVert C_n\rVert_2^2=-4\langle E_n,F_n\rangle.
 \label{eq:contrastenergy}
\end{equation}
Opposing half totals can therefore produce a large detail relative to the cumulative source, while reinforcing totals can produce a small detail. The matching operation in~\eqref{eq:rmsmatch} controls the magnitude presented to the router after this contrast has been formed.

\Needspace{11\baselineskip}
\subsection{Resolution along the block}

\begin{proposition}[Resolution under permutation]
For a fixed realized sequence of outputs, $C_n$ is unchanged by any permutation. The raw pair $(C_n,D_n)$ is unchanged by permutations within either half. Exchanging an early output $v$ with a late output $w$ leaves $C_n$ unchanged and changes $D_n$ to $D_n+2(w-v)$.
\end{proposition}
\begin{proof}
Within-half permutations preserve each sum in~\eqref{eq:segments}. A cross-boundary exchange changes the totals to $E_n-v+w$ and $F_n-w+v$. Adding these expressions preserves $C_n$; subtracting them gives $D_n+2(w-v)$.
\end{proof}

The split resolves allocation across one boundary. Within each half, outputs with the same sum remain identified. For $m>2$, the raw representation therefore compresses the trajectory even though it recovers both half totals. At $m=2$, each half has one output and the pair recovers both individual outputs. At $m=1$, the late total is empty and $D=C$, adding no raw information.

Two other cases distinguish the sum from the contrast. Equal half totals give $D=0$, even when the individual outputs within either half vary. Opposite half totals give $C=0$ and possibly nonzero $D$. The first case places all retained information in the cumulative coordinate; the second places it in the detail. These algebraic cases also explain why absence of a detail slot differs from a slot whose current value happens to be zero.

The temporal resolution is fixed by $m$ and the midpoint, rather than learned for each example. This gives the extra source a common interpretation across blocks and introduces no boundary-selection parameters. It also makes granularity a substantive architectural choice: a half split of a 24-event block summarizes different portions of the computation from a half split of a 12-event block. The block-count ablation examines that choice at fixed network depth.

\subsection{From raw coordinates to routed mixtures}

Exact recovery applies to $(C,D)$. The implemented read uses $(C,\widehat D)$ together with other sources and nonnegative softmax weights. It therefore operates on a different object from the unrestricted linear combination in~\eqref{eq:span}. Magnitude matching changes the supplied coordinate, and the scale multiplier is not an additional source given to the router.

Within a fixed RMS reduction group, write $\widehat D=aD$, with $a>0$. At each token and coordinate in that group, if the router assigns weights $\alpha_C$ and $\alpha_D$ to this block's two slots, their contribution is
\begin{equation}
 \alpha_C C+\alpha_D\widehat D
 =(\alpha_C+a\alpha_D)E+(\alpha_C-a\alpha_D)F.
 \label{eq:halfmix}
\end{equation}
With only $C$, the two half totals receive the same coefficient. A positively weighted detail separates those coefficients and can make the late coefficient negative. The available direction is oriented: for this isolated pair, the early coefficient is at least the late coefficient. A softmax over $(C,\widehat D)$ does not provide arbitrary signed weights over $(E,F)$. The transformation changes the geometry seen by the router instead of merely renaming two interchangeable slots.

The other blocks and the embedding enter through the same normalization. Let $\mathcal C$ contain their cumulative slots and the embedding, and let $\mathcal D$ contain the detail slots. Set $Z_C=\sum_{s\in\mathcal C}\exp z_s$ and $Z_D=\sum_{s\in\mathcal D}\exp z_s$. For nonempty groups, let $\mu_C$ and $\mu_D$ be their within-group normalized mixtures. Then
\begin{equation}
 \widetilde h=(1-p_D)\mu_C+p_D\mu_D,\qquad
 p_D=\frac{Z_D}{Z_C+Z_D}.
 \label{eq:groupmixture}
\end{equation}
Details affect a read both through their values and through competition for normalized weight. Even a zero-valued detail can change the coefficient of $\mu_C$ by changing $Z_D$. This separates an added direction from an added routing slot and motivates the duplicate-source control.

Masking all detail slots gives the cumulative-only read exactly. With fixed finite content scores, taking their shared bias to $-\infty$ has the same limit. The initialization at $-2$ is a finite soft preference, so HAARES starts with nonzero detail weight. The relation is a property of a fixed read; during training, the different reads produce different later source values.

\subsection{The duplicate-source control}

Suppose two slots contain the same value $c$ and the same normalized key. Let $z$ be their common content score and $b_1,b_2$ their biases. Their combined softmax numerator is
\begin{equation}
 \bigl(\exp(z+b_1)+\exp(z+b_2)\bigr)c
 =\exp\!\bigl(z+\log(\exp b_1+\exp b_2)\bigr)c.
 \label{eq:duplicate}
\end{equation}
The denominator combines in the same way. At that read, duplicating the source is equivalent to retaining one value with an effective log-sum-exp bias. It adds no raw direction, but it changes the source's allocation relative to other slots. Trainable duplicate biases also change the optimization parameterization.

This control is especially relevant to HAARES because its active raw pair coincides during the early half. Before late outputs arrive, the extra slot can already affect allocation. Later in the block, it carries the early--late contrast as well. Comparing HAARES only with the cumulative-only parent would combine these two changes. The duplicated-$C$ experiment tests the extra-slot alternative, while the random-sign detail tests the particular structure of the added projection. Their measured endpoints are analyzed in Section~\ref{sec:ablations}.

\subsection{Source-count complexity}
\label{sec:complexity}

Before event $r$ of block $n$, the two source counts are
\begin{equation}
 S_{\mathrm{Block}}(n,r)=n+\mathbf1_{r>1},\qquad
 S_{\method}(n,r)=2n-1+2\mathbf1_{r>1}.
\end{equation}
Averaging over the $N$ blocks and $m$ events per block, excluding the final readout, gives
\begin{equation}
 \overline S_{\mathrm{Block}}=\frac{N+3}{2}-\frac1m,\qquad
 \overline S_{\method}=N+2-\frac2m.
 \label{eq:avgsource}
\end{equation}
For $L=48,N=4$, the averages are 3.46 and 5.92, with maxima 5 and 9. Both remain $O(N)$ in source count; HAARES changes the constant rather than recovering the $2L$ individual outputs. Scoring and mixing cost $O(SBTd)$ for $S$ slots. With $d_{\mathrm{ff}}$ proportional to $d$, the extra arithmetic relative to the main projections scales approximately as $O(N/d)$. The estimate counts arithmetic; Section~\ref{sec:efficiency} measures the accompanying memory and runtime costs.

The count follows directly from the online interface. At the first event of block $n$, the embedding and $n-1$ completed cumulative sources give $n$ baseline slots. HAARES instead has the embedding and two slots for each completed block, giving $2n-1$. The remaining $m-1$ events also see one active cumulative source or an active pair. Thus
\begin{equation}
 \frac1{Nm}\sum_{n=1}^{N}\sum_{r=1}^{m}\mathbf1_{r>1}
 =1-\frac1m,
\end{equation}
which yields~\eqref{eq:avgsource}. For $m>1$, the maxima occur in the last block after its first event. Table~\ref{tab:sourcecounts} gives the counts for the block granularities used in the experiments. These are exact source counts from the source-bank definition.

\begin{table}[tbp]
\centering
\caption{Analytical intermediate-read source counts at 48 layers. The final router is excluded from the average and uses $N+1$ sources in both methods. An event is an attention or MLP sublayer output.}
\label{tab:sourcecounts}
\small
\begin{tabular}{@{}rrrrrr@{}}
\toprule
$N$ & Events/block $m$ & \multicolumn{2}{c}{Block AttnRes} & \multicolumn{2}{c}{HAARES}\\
\cmidrule(lr){3-4}\cmidrule(l){5-6}
 & & Average & Maximum & Average & Maximum\\
\midrule
4 & 24 & 3.46 & 5 & 5.92 & 9\\
6 & 16 & 4.44 & 7 & 7.88 & 13\\
8 & 12 & 5.42 & 9 & 9.83 & 17\\
\bottomrule
\end{tabular}
\end{table}

These counts expose a trade-off that the phrase ``one extra source per block'' can obscure. Later sublayers repeatedly read the completed details, so the total work is not just the cost of forming $D$ once. Conversely, retaining two summaries per block is still different from keeping every sublayer output as a distinct source. At fixed $N$, increasing $L$ lengthens the sequence compressed by each pair without increasing its maximum source count. At fixed $L$, increasing $N$ preserves finer cumulative resolution but increases the number of candidates read downstream. Varying the block count changes both the compression within a source and the work required to read the bank.

\FloatBarrier
\section{Experimental setup}
\label{sec:methodology}

We compare HAARES with its cumulative-only parent under the same backbone and training recipe. The depth sweep varies how many outputs each of four blocks compresses. The 48-layer comparisons then examine model seeds, width, and a BPE pipeline. Four additional corpora test the same architectural choice in separately trained language models. Finally, granularity, duplicate-source, random-sign, scale, and bias ablations examine which parts of the change account for the loss differences.

The primary outcome is the best validation cross-entropy reached within a fixed training budget. Seed-level values are available for the principal 48-layer comparisons and ablations; the domain results are reported as seed means. Every task is next-token prediction, including the review and code datasets.

\subsection{Datasets and tokenization}

OpenWebText is the main corpus. The additional datasets are TinyStories, WikiText-103, Yelp Review Full, and CodeSearchNet-Python \citep{gokaslan2019openwebtext,eldan2023tinystories,merity2017wikitext,zhang2015charcnn,husain2019codesearchnet}. Table~\ref{tab:data} gives their configurations and approximate sizes; Appendix~\ref{app:data} records dataset revisions and preprocessing details.

The character tokenizer has 256 entries: four special tokens and the 252 most frequent training-corpus characters. Text is processed line by line, without lowercasing or additional filtering. The pipeline forms length-513 chunks for 512-token next-token prediction. It neither concatenates separate input lines before chunking nor inserts an extra end-of-sequence symbol between lines. Code indentation and line content are retained, with carriage returns and newlines normalized consistently.

The BPE experiment uses a 4096-entry byte-level tokenizer \citep{sennrich2016bpe}, trained through the Hugging Face \texttt{tokenizers} backend on the first 100M raw OpenWebText bytes. Tokenization processes approximately 14.3B raw bytes in 1M-byte batches and stores about 7.8B tokens as 16-bit integers, or roughly 1.83 raw bytes per token. This pipeline uses a random 95/5 train/validation split, rather than the character pipeline's row-based split. We compare architectures within each pipeline and keep character and BPE losses separate because their token units and data splits differ.

\begin{table}[tbp]
\centering
\caption{Dataset sources and approximate training/evaluation sizes. The evaluation column lists the split used for checkpoint selection.}
\label{tab:data}
\small
\begin{tabularx}{\textwidth}{@{}lXrr@{}}
\toprule
Dataset & Source/configuration & Train & Validation \\
\midrule
OpenWebText & Hugging Face \texttt{openwebtext}; 30 of 80 parquet shards; row index modulo 10 defines validation & $\sim$2.7M lines, 13.4 GB & $\sim$301K lines, 1.5 GB \\
WikiText-103 & \texttt{wikitext-103-raw-v1} & $\sim$1.8M lines, 537 MB & 3,760 lines, 1.1 MB \\
TinyStories & \texttt{roneneldan/TinyStories} & $\sim$2.1M lines, 1.9 GB & $\sim$22K lines, 19 MB \\
Yelp Review Full & Hugging Face \texttt{yelp\_review\_full} & 650,000 lines, 475 MB & 50,000 lines, 37 MB \\
CodeSearchNet-Python & Hugging Face \texttt{code\_search\_net}, subset \texttt{python} & 412,178 lines, 342 MB & 23,107 lines, 20 MB \\
\bottomrule
\end{tabularx}
\end{table}

\subsection{Backbone, model families, and baselines}

All models are causal decoder-only PreNorm Transformers. They use affine RMSNorm with $\epsilon=10^{-6}$ \citep{zhang2019rmsnorm}, rotary positional embeddings with base 10,000 and cached maximum length 2048 \citep{su2024roformer}, SwiGLU MLPs \citep{shazeer2020glu}, eight attention heads, and tied input/output embeddings \citep{press2017tied}. QKV and attention-output projections are bias-free, and dropout is zero. Attention is evaluated through PyTorch's \texttt{scaled\_dot\_product\_attention}. The selected backend kernel is not identified by the recorded API call.

Linear and embedding weights are initialized from $\mathcal N(0,0.02^2)$; ordinary biases and router queries start at zero. Table~\ref{tab:modelconfig} lists the model families. Unless explicitly varied, the block partition has $N=4$. In a 48-layer model this gives 24 residual events per block. The extra BPE vocabulary changes the medium model's parameter count from approximately 201M to 204M, while its hidden width, depth, and MLP width remain fixed.

\begin{table}[tbp]
\centering
\caption{Model families. Parameters are approximate. The 4096-entry BPE vocabulary raises the medium 48-layer model from about 201M to 204M parameters.}
\label{tab:modelconfig}
\small
\begin{tabular}{@{}lccc@{}}
\toprule
Configuration & Small & Medium & Large \\
\midrule
Hidden width $d$ & 128 & 512 & 768 \\
MLP width $d_{\mathrm{ff}}$ & 1024 & 2048 & 3072 \\
Layers $L$ & 12/24/48 & 12/24/48 & 48 \\
Attention heads & 8 & 8 & 8 \\
Head dimension & 16 & 64 & 96 \\
Context length & 512 & 512 & 512 \\
Vocabulary & 256 & 256 or 4096 & 256 \\
Approx. parameters at 48L & 22M & 201M or 204M & 453M \\
Default blocks $N$ & 4 & 4 & 4 \\
\bottomrule
\end{tabular}
\end{table}

Block AttnRes is the direct parent comparison: it removes detail slots while retaining the same cumulative sources and routing implementation. Standard Residual uses ordinary additive PreNorm connections. The small-family domain experiments also include ReZero, with a zero-initialized scalar per residual sublayer, and LayerScale, with per-channel initial scales of $0.1$, $10^{-5}$, and $10^{-6}$ at 12, 24, and 48 layers. These baselines establish performance under the common recipe; they are not independently retuned for each dataset.

The full HAARES configuration uses the fixed half split, RMS matching with $\epsilon=10^{-6}$ and $\gamma=4$, and learned detail biases initialized at $-2$. Its final readout remains cumulative-only. The ablations change the named source, scale, bias, or partition component and retain the common model and training configuration.

\subsection{Optimization and model selection}

Training uses fp32 and AdamW \citep{loshchilov2019decoupled} with learning rate $3\times10^{-4}$, $\beta_1=0.9$, $\beta_2=0.95$, optimizer $\epsilon=10^{-8}$, and weight decay 0.01. Biases, normalization parameters, gates, residual scales, router queries, and parameters with fewer than two dimensions are excluded from weight decay. The learning rate is constant, without warmup or decay. Global gradient norm is clipped at 1.0. Batch size is 16 per GPU, with context length 512 and 8192 target tokens per step. Training uses no gradient accumulation, activation checkpointing, early stopping, or \texttt{torch.compile}.

Small-family 12- and 24-layer models train for 50K steps. Small-family 48-layer models and all medium/large models train for 30K steps. Validation occurs every 2K steps, with one complete pass over the evaluation split. For method $M$, model seed $j$, and scheduled evaluations $\mathcal T$, the reported endpoint and paired gain are
\begin{equation}
 \ell_{M,j}^{\star}=\min_{t\in\mathcal T}\ell_{M,j}(t),\qquad
 g_j=\ell_{\mathrm{Block},j}^{\star}-\ell_{\method,j}^{\star}.
 \label{eq:endpoint}
\end{equation}
Positive $g_j$ favors HAARES. Each method can reach its minimum at a different checkpoint. This endpoint answers a budgeted model-selection question: which validation quality was attained during the run? It is distinct from a final-step comparison and from evaluation on an untouched test split.

Model seeds are 42, 123, and 2026; the 453M comparison uses 42 and 2026. The data-order seed is fixed at 42. Pairing methods by model seed and data order keeps these choices common while allowing the architectures to follow their own training trajectories. We report individual endpoints, arithmetic means, and sample standard deviations where the individual values are available. The paired gains remain visible alongside method-wise dispersion, so a mean difference cannot hide an unfavorable seed. These summaries describe the observed model initializations; no significance tests or cross-setting pooled confidence intervals are used.

Perplexity is the exponential of cross-entropy within a given tokenization. A paired loss reduction $g_j$ corresponds to a perplexity ratio $\exp(-g_j)$. Tables preserve the reported rounded perplexities, apart from three explicitly identified arithmetic corrections in Appendix~\ref{app:pairsppl}. Aggregate perplexities are retained as reported aggregates rather than recomputed by exponentiating a rounded mean loss. This matters because averaging perplexities and exponentiating an average loss are different operations.

Architecture and checkpoint choices use validation outcomes. Domain means are therefore interpreted within each dataset and depth; the BPE results compare two methods under the same BPE pipeline. Details not fully specified in the experiment record---including the BPE split unit and seed, tokenizer-training overlap with validation, and RMS reduction axes---are collected in Appendix~\ref{app:repro}. They are kept separate from the parameters that the record does specify.

\subsection{Hardware and cost protocol}
\label{sec:costprotocol}

Small-family accuracy experiments use 24GB NVIDIA GeForce RTX 4090 GPUs. Medium/large experiments and resource measurements use 96GB NVIDIA RTX PRO 6000 Blackwell GPUs. The recorded environment is Ubuntu 22.04, Python 3.12, PyTorch 2.8.0, CUDA 12.8, and cuDNN 9.x.

Resource measurements use one GPU, fp32, batch size 16, and context length 512. Throughput is the number of processed tokens divided by elapsed time over 2K steps, beginning at step zero and including startup and initial data-pipeline overhead. Peak memory is allocated memory from \texttt{torch.cuda.max\_memory\_allocated()}, rather than reserved memory. We derive step times from the same rounded throughputs used in the budget comparison.

We report analytical arithmetic, peak allocated memory, and measured throughput separately. Additional source operations scale differently from the backbone projections and can incur memory-access costs disproportionate to their arithmetic count. This distinction also motivates IO-aware attention implementations \citep{dao2022flashattention}. The measurements here characterize the unfused source-handling path actually used, with no extrapolation to a different precision or distributed implementation.

\section{Results}

\subsection{Depth dependence on OpenWebText}
\label{sec:depth}

The clearest gain occurs at 48 layers. With four blocks and width $d=512$, HAARES lowers mean validation loss from 1.8057 to 1.7744, a reduction of 0.0313 (Table~\ref{tab:depthsweep}). At 12 layers the reduction is 0.0078; at 24 layers HAARES is slightly worse, by 0.0016. The reduction is largest at 48 layers, but it does not grow monotonically with depth.

\begin{table}[tbp]
\centering
\caption{OpenWebText depth sweep, $d=512$, $d_{\mathrm{ff}}=2048$, $N=4$. Entries are mean best validation loss / reported perplexity over three seeds. $\Delta=\ell_{\method}-\ell_{\mathrm{Block}}$; negative favors HAARES.}
\label{tab:depthsweep}
\small
\begin{tabular}{@{}lccc@{}}
\toprule
Depth & Block AttnRes & \method{} & $\Delta$ loss \\
\midrule
12L & $1.8527/6.38$ & $\mathbf{1.8449}/\mathbf{6.33}$ & $-0.0078$ \\
24L & $\mathbf{1.8112}/\mathbf{6.12}$ & $1.8128/6.13$ & $+0.0016$ \\
48L & $1.8057/6.09$ & $\mathbf{1.7744}/\mathbf{5.90}$ & $-0.0313$ \\
\bottomrule
\end{tabular}
\end{table}

A fixed four-block partition contains 6, 12, and 24 residual events per block at these depths. The 48-layer model asks each cumulative source to summarize four times as many outputs as the 12-layer model. HAARES retains the same two coordinates per block in all three cases, so its additional source preserves an early--late distinction over progressively longer segments. The strongest gain at the longest segment is consistent with that motivation.

The intermediate 24-layer result is important for interpreting this trend. Its cumulative-only baseline already improves substantially over the 12-layer baseline, but adding the detail changes little and reverses the sign. The 24-layer result rules out a simple monotonic relationship between depth and gain. The fixed-depth block-count experiment in Section~\ref{sec:ablation} tests the related granularity question without changing the network's number of layers or parameters.

\subsection{Paired seeds, width, and tokenization}
\label{sec:paired}

HAARES improves every paired seed in the 201M, 48-layer comparison. The Block-minus-HAARES reductions are 0.0168, 0.0322, and 0.0449 for seeds 42, 123, and 2026, respectively. Mean loss changes from $1.8057\pm0.0136$ to $1.7744\pm0.0095$, with sample standard deviations calculated over the three run minima (Table~\ref{tab:201m}). Both the mean and the individual pairs favor the added source.

\begin{table}[tbp]
\centering
\caption{OpenWebText, 201M parameters, 48 layers. Seed cells report best validation loss / perplexity. The final column is the mean and sample standard deviation of loss. Bold identifies the lowest seed or mean loss.}
\label{tab:201m}
\small
\begin{tabular}{@{}lcccc@{}}
\toprule
Method & Seed 42 & Seed 123 & Seed 2026 & Mean loss $\pm$ SD \\
\midrule
Standard Residual & $1.8122/6.12$ & $1.8570/6.40$ & $1.8287/6.23$ & $1.8326\pm0.0227$ \\
Block AttnRes & $1.7995/6.05$ & $1.7963/6.03$ & $1.8213/6.18$ & $1.8057\pm0.0136$ \\
\method{} & $\mathbf{1.7827}/\mathbf{5.95}$ & $\mathbf{1.7641}/\mathbf{5.84}$ & $\mathbf{1.7764}/\mathbf{5.91}$ & $\mathbf{1.7744}\pm0.0095$ \\
\bottomrule
\end{tabular}
\end{table}

The Standard Residual rows separate the inherited benefit of block routing from the benefit of the half split. Standard Residual has mean loss 1.8326. The reduction to 1.8057 belongs to the cumulative-only parent; the further reduction to 1.7744 comes from the HAARES comparison. This distinction is useful because a single comparison against ordinary residual connections would combine the two architectural changes. HAARES is also below Standard Residual in each individual seed.

The 453M model keeps 48 layers and increases hidden width from 512 to 768. HAARES again improves both available paired seeds, by 0.0383 and 0.0082, giving a mean reduction of 0.0233 after rounding (Table~\ref{tab:453m}). The improvement persists at the larger width, with seed 42 contributing most of the reduction. HAARES has greater endpoint dispersion than the parent in these two runs. The wider model benefits from the source addition, but its mean gain is smaller than in the 201M comparison.

\begin{table}[tbp]
\centering
\caption{OpenWebText, 453M parameters, 48 layers. Entries are best validation loss / perplexity. The comparison uses two model seeds; the last column gives their mean and sample standard deviation.}
\label{tab:453m}
\small
\begin{tabular}{@{}lccc@{}}
\toprule
Method & Seed 42 & Seed 2026 & Mean loss $\pm$ SD \\
\midrule
Standard Residual & $1.8280/6.22$ & $1.8315/6.24$ & $1.8298\pm0.0025$ \\
Block AttnRes & $1.7971/6.03$ & $1.7945/6.02$ & $1.7958\pm0.0018$ \\
\method{} & $\mathbf{1.7588}/\mathbf{5.81}$ & $\mathbf{1.7863}/\mathbf{5.97}$ & $\mathbf{1.7726}\pm0.0194$ \\
\bottomrule
\end{tabular}
\end{table}

The BPE model tests the same source change under a 4096-entry tokenization pipeline (Table~\ref{tab:bpe}). The three paired reductions are 0.0086, 0.0315, and 0.0049, with mean 0.0150. All three are positive, including the two smaller improvements. The effect is thus present in both the character and BPE experiments. The BPE pipeline also changes the split procedure, so the smaller mean gain cannot be attributed to tokenization alone.

\begin{table}[tbp]
\centering
\caption{OpenWebText with 4096-entry BPE, 204M parameters, 48 layers. Seed cells report best validation loss / perplexity; the last column reports mean loss and sample standard deviation.}
\label{tab:bpe}
\small
\begin{tabular}{@{}lcccc@{}}
\toprule
Method & Seed 42 & Seed 123 & Seed 2026 & Mean loss $\pm$ SD \\
\midrule
Standard Residual & $1.8910/6.63$ & $1.8896/6.62$ & $1.8878/6.60$ & $1.8895\pm0.0016$ \\
Block AttnRes & $1.8692/6.48$ & $1.8670/6.47$ & $1.8651/6.46$ & $1.8671\pm0.0021$ \\
\method{} & $\mathbf{1.8606}/\mathbf{6.43}$ & $\mathbf{1.8355}/\mathbf{6.27}$ & $\mathbf{1.8602}/\mathbf{6.43}$ & $\mathbf{1.8521}\pm0.0144$ \\
\bottomrule
\end{tabular}
\end{table}

Figure~\ref{fig:pairedevidence} displays the individual gains together. The 201M character comparison spans 0.0168--0.0449; the 453M comparison spans 0.0082--0.0383; and BPE spans 0.0049--0.0315. All eight pairs improve, but their magnitudes vary substantially. The individual points show how much the gain varies across the available seeds. Each horizontal range refers to its own model and tokenization; Appendix~\ref{app:pairs} gives the exact plotted values.

\begin{figure}[tbp]
\centering
\includegraphics[width=0.95\textwidth]{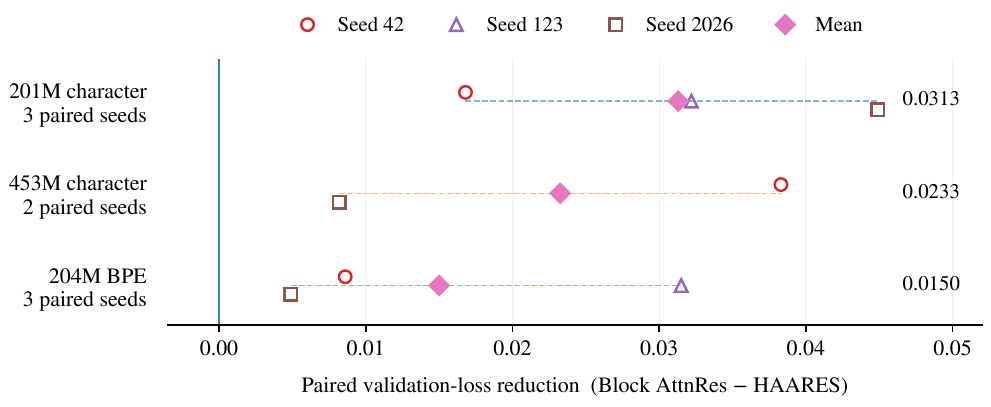}
\caption{Paired 48-layer reductions in best validation loss, $\ell_{\mathrm{Block}}-\ell_{\method}$. Each symbol denotes a model seed; diamonds mark setting means. Horizontal segments span the observed minimum and maximum and are not confidence intervals. The zero line marks equal loss. Character and BPE losses have different token units, so magnitudes should be interpreted within each setting}
\label{fig:pairedevidence}
\end{figure}

These comparisons give two complementary observations. The gain recurs across the recorded initializations, so it is not carried by a single successful run. Its variability increases in some settings: the 453M and BPE HAARES endpoints are more dispersed than their parent endpoints, unlike the 201M character case. We therefore use the paired direction to establish repeatability within the reported runs and the individual magnitudes to describe how strongly each run benefits.

\subsection{Separately trained domain benchmarks}
\label{sec:crossdomain}

At 48 layers, HAARES achieves the lowest mean validation loss among all five methods on each additional corpus (Table~\ref{tab:crossdomain}). The reductions over Block AttnRes are 0.0188 on TinyStories, 0.0135 on WikiText-103, 0.0133 on Yelp Review Full, and 0.0224 on CodeSearchNet-Python. These models use the small family, with hidden width 128, so the 48-layer result extends across both corpus and model width.

\begin{table}[tbp]
\centering
\caption{Cross-domain mean best validation loss, small family ($d=128$, $d_{\mathrm{ff}}=1024$). Bold marks the smallest reported mean at each dataset/depth, including ties. Seed-level variability is not available for these means. Each model is trained separately on its dataset.}
\label{tab:crossdomain}
\small
\begin{tabular}{@{}llccc@{}}
\toprule
Dataset & Method & 12L & 24L & 48L \\
\midrule
\multirow{5}{*}{TinyStories}
& Standard Residual & 0.5555 & 0.5391 & 0.5219 \\
& ReZero & 0.5815 & 0.5571 & 0.5520 \\
& LayerScale & 0.5663 & 0.5614 & 0.5446 \\
& Block AttnRes & \textbf{0.5549} & \textbf{0.5310} & 0.5355 \\
& \method{} & 0.5569 & 0.5313 & \textbf{0.5167} \\
\midrule
\multirow{5}{*}{WikiText-103}
& Standard Residual & 1.0337 & 1.0014 & 0.9783 \\
& ReZero & 1.0952 & 1.0533 & 1.0312 \\
& LayerScale & 1.0421 & 1.0624 & 1.0300 \\
& Block AttnRes & 1.0325 & 0.9990 & 0.9742 \\
& \method{} & \textbf{1.0258} & \textbf{0.9754} & \textbf{0.9607} \\
\midrule
\multirow{5}{*}{Yelp Review Full}
& Standard Residual & \textbf{0.9945} & \textbf{0.9713} & 0.9605 \\
& ReZero & 1.0095 & 0.9879 & 0.9802 \\
& LayerScale & 0.9977 & 0.9874 & 0.9710 \\
& Block AttnRes & 0.9990 & 0.9988 & 0.9703 \\
& \method{} & 1.0053 & \textbf{0.9713} & \textbf{0.9570} \\
\midrule
\multirow{5}{*}{CodeSearchNet-Python}
& Standard Residual & 0.9306 & 0.9019 & 0.8861 \\
& ReZero & 0.9455 & 0.9094 & 0.8925 \\
& LayerScale & 0.9277 & 0.9139 & 0.8863 \\
& Block AttnRes & 0.9242 & 0.8918 & 0.8712 \\
& \method{} & \textbf{0.9095} & \textbf{0.8657} & \textbf{0.8488} \\
\bottomrule
\end{tabular}
\end{table}

The shallower comparisons differ by dataset. WikiText-103 and CodeSearchNet-Python favor HAARES at all three depths. TinyStories favors Block AttnRes at 12 and 24 layers, by 0.0020 and 0.0003, before favoring HAARES at 48 layers. On Yelp, the 12-layer HAARES model is worse than the parent, while the 24-layer HAARES model matches Standard Residual to the reported precision and improves over Block AttnRes. HAARES has lower mean loss than Block AttnRes on two of four corpora at 12 layers, three at 24 layers, and all four at 48 layers.

The Standard Residual comparison also varies across these tasks. At 48 layers, Block AttnRes is worse than Standard Residual on TinyStories and Yelp, whereas HAARES improves over both. On WikiText-103 and CodeSearchNet-Python, the parent already improves on Standard Residual and the half split provides a further reduction. On TinyStories and Yelp, the detail improves a parent that does not itself outperform Standard Residual; on the other two corpora, it extends the parent's gain.

Each row family is a separately trained next-token model, not a transfer evaluation of one shared checkpoint. The domain table records seed means without individual dispersion, and the small-family budgets are 50K steps at 12/24 layers and 30K at 48 layers. Comparisons within a corpus and depth are matched; the cross-depth pattern should be read with those budgets in view. Reported perplexities are retained in Appendix~\ref{app:crossppl}.

Across the main and additional corpora, the four-block 48-layer configuration is the most consistent setting for HAARES. The next section examines whether that consistency follows simply from exposing more sources, or whether the half-split representation and its routing treatment matter.

\FloatBarrier
\section{Ablations of granularity and source design}
\label{sec:ablations}

\subsection{Changing block granularity}
\label{sec:ablation}

At fixed depth and width, increasing the number of blocks changes both the length compressed by each summary and the number of candidates read downstream. For the 48-layer, 201M model, $N=4,6,8$ gives 24, 16, and 12 residual events per block. The mean HAARES-minus-Block loss is $-0.0313$ with four blocks, but $+0.0069$ with six and $+0.0113$ with eight (Table~\ref{tab:blockablation}). Of these three partitions, only the four-block setting has a lower HAARES mean than its matched parent.

\begin{table}[tbp]
\centering
\caption{Block-count ablation on the 201M, 48-layer OpenWebText model. Entries are best validation losses. $\Delta$ is the mean HAARES loss minus the Block AttnRes mean at the same $N$.}
\label{tab:blockablation}
\small
\begin{tabular}{@{}lcccccc@{}}
\toprule
Method & $N$ & Seed 42 & Seed 123 & Seed 2026 & Mean & $\Delta$ \\
\midrule
Block AttnRes & 4 & $1.7995$ & $1.7963$ & $1.8213$ & $1.8057$ & -- \\
\method{} & 4 & $1.7827$ & $1.7641$ & $1.7764$ & $\mathbf{1.7744}$ & $\mathbf{-0.0313}$ \\
Block AttnRes & 6 & $1.7957$ & $1.7770$ & $1.7880$ & $\mathbf{1.7869}$ & -- \\
\method{} & 6 & $1.7941$ & $1.7723$ & $1.8151$ & $1.7938$ & $+0.0069$ \\
Block AttnRes & 8 & $1.7857$ & $1.7957$ & $1.7961$ & $\mathbf{1.7925}$ & -- \\
\method{} & 8 & $1.8222$ & $1.7933$ & $1.7959$ & $1.8038$ & $+0.0113$ \\
\bottomrule
\end{tabular}
\end{table}

The individual seeds reveal how the mean reversal arises. At $N=6$, HAARES improves by 0.0016 and 0.0047 in seeds 42 and 123, but worsens by 0.0271 in seed 2026. At $N=8$, it worsens by 0.0365 in seed 42, improves by 0.0024 in seed 123, and is nearly tied in seed 2026, with a 0.0002 reduction. At each finer partition, the loss increase in one seed outweighs the smaller reductions in the other two. Four blocks is the only tested partition with improvements in all three seeds.

The cumulative-only model benefits from a finer partition up to six blocks: its mean falls from 1.8057 at $N=4$ to 1.7869 at $N=6$, then rises to 1.7925 at $N=8$. HAARES with four blocks still reaches the lowest mean, 1.7744, improving on the best cumulative-only setting by 0.0125. Increasing cumulative resolution alone does not reproduce that endpoint. The comparison is based on the same validation selection criterion used throughout the study.

A useful structural comparison is HAARES at $N=4$ against Block AttnRes at $N=8$. Both have at most nine intermediate source slots, but their source coordinates differ. For a fixed realized 24-event coarse block, the two 12-event cumulative summaries are its half totals. They are related to the HAARES coordinates by
\begin{equation}
 (E_n,F_n)\quad\longleftrightarrow\quad(C_n,D_n)=(E_n+F_n,E_n-F_n).
 \label{eq:granularitybasis}
\end{equation}
At completed coarse-block boundaries, these raw representations retain the same pair of half totals in different coordinates. Their routers nevertheless see different key directions and different candidate values. Equation~\eqref{eq:halfmix} shows the coefficient geometry induced by a positively weighted detail; routing directly over $E$ and $F$ has a different nonnegative mixture geometry.

The online interfaces differ as well. HAARES retains the same coarse block identity as it moves from early accumulation to late subtraction. The finer cumulative partition completes an early block and opens a new active block at the midpoint. Their average intermediate source counts are 5.92 and 5.42, respectively, and their final banks contain five and nine sources. Equal peak slot count therefore does not make the two algorithms equivalent. The granularity experiment compares practical residual designs whose coordinates, availability, and readout differ together.

The depth sweep and this fixed-depth ablation agree on the most favorable configuration: a half-split detail added to a relatively long block. They also show why simply adding more source slots is not a reliable rule. The retained distinction and the routing problem created by those slots have to be considered together.

\subsection{Source structure, RMS matching, and detail bias}

The source controls distinguish an extra weighted slot from an early--late contrast (Table~\ref{tab:basisablation}). Duplicating the cumulative source lowers mean loss from 1.8057 to 1.7848. Full HAARES lowers it further to 1.7744 and beats the duplicate in each seed, by 0.0124, 0.0117, and 0.0070. The added slot explains a substantial improvement over the parent, but the structured detail gives a further consistent benefit.

\begin{table}[!htbp]
\centering
\caption{Source and stabilization ablations, 201M parameters, 48 layers, $N=4$. All entries are best validation loss. The random-sign variant uses fixed signs. All other named components follow the shared configuration.}
\label{tab:basisablation}
\small
\begin{tabular}{@{}lcccc@{}}
\toprule
Variant & Seed 42 & Seed 123 & Seed 2026 & Mean \\
\midrule
Block AttnRes & $1.7995$ & $1.7963$ & $1.8213$ & $1.8057$ \\
Duplicated $C\times2$ & $1.7951$ & $1.7758$ & $1.7834$ & $1.7848$ \\
Random-sign detail & $1.8100$ & $1.8000$ & $1.8080$ & $1.8060$ \\
Fixed detail bias $0$ & $1.8040$ & $1.7941$ & $1.8090$ & $1.8024$ \\
Fixed detail bias $-2$ & $1.8172$ & $1.7816$ & $1.7895$ & $1.7961$ \\
Fixed detail bias $-4$ & $1.8226$ & $1.8172$ & $1.8291$ & $1.8230$ \\
\method{} without RMS match & $1.7862$ & $1.7747$ & $1.7782$ & $1.7797$ \\
\method{} & $\mathbf{1.7827}$ & $\mathbf{1.7641}$ & $\mathbf{1.7764}$ & $\mathbf{1.7744}$ \\
\bottomrule
\end{tabular}
\end{table}

The duplicate result agrees with~\eqref{eq:duplicate}: two equal values can change softmax allocation even without adding a raw direction. Duplication is an active reweighting control, not a change that leaves the parent read unchanged. The 0.0104 mean advantage of HAARES over duplication measures the improvement over that reweighting alternative. Since the variants are trained separately, these differences compare complete training conditions rather than additive contributions that must sum to a unique causal decomposition.

The fixed random-sign detail has mean loss 1.8060, essentially the cumulative-only result. Replacing the midpoint sign pattern with that control thus removes the improvement. The experiment supports the chosen ordered split over the tested random projection; the record does not specify the random sign realization or include a sweep over sign patterns.

Removing RMS matching yields mean loss 1.7797, retaining a 0.0260 reduction from Block AttnRes. Matching provides a smaller additional reduction of 0.0053. Its paired improvements are 0.0035, 0.0106, and 0.0018; the reduction is present in every seed. The structured source remains useful without scale matching, while bounded matching improves its integration into the router. This interpretation keeps the source itself separate from the treatment of its magnitude.

Learning the detail bias also matters. Fixing it at $0$, $-2$, or $-4$ gives mean losses 1.8024, 1.7961, and 1.8230, all above the learned-bias result. The fixed $-2$ variant shares the full method's initial preference but cannot adapt that preference during training. The more negative fixed $-4$ variant performs worse than the cumulative-only parent. None of the tested fixed biases matches the learned bias, and a finite negative bias does not remove the detail slots.

The latter comparison is consistent with the router definition. A finite bias leaves detail mass in the denominator and permits nonzero detail values to affect subsequent computations. Only masking, or a limiting bias of $-\infty$ at a fixed read, recovers the cumulative-only mixture. Across the ablations, the best endpoint comes from combining a structured contrast with bounded value scaling and a learned allocation bias. No single control reproduces the full result.

\FloatBarrier
\section{Computational cost and quality targets}

\subsection{Measured resource cost}
\label{sec:efficiency}

HAARES adds less than 0.01\% to the reported trainable parameter count. Its practical cost is dominated by additional activation sources rather than model weights. Table~\ref{tab:cost} gives peak allocated memory, throughput, and the reported analytical FLOP estimates for three 48-layer widths. The FLOP increments are approximately 0.5\%, 0.3\%, and 0.2\%, while the measured memory and time penalties are larger.

\begin{table}[!htbp]
\centering
\caption{Matched 48-layer cost on one RTX PRO 6000 Blackwell, fp32, batch size 16, context length 512. Each pair is Block AttnRes $\to$ HAARES. Memory and throughput are measured values; step times and relative increases are calculated from those rounded measurements. FLOP increments are analytical estimates.}
\label{tab:cost}
\small
\setlength{\tabcolsep}{5pt}
\begin{tabular}{@{}rccccc@{}}
\toprule
$d$ & Peak memory & Throughput & Step time & Memory / time inc. & FLOP inc. \\
 & (GB) & (K tokens/s) & (s) & (\%) & (\%) \\
\midrule
128 & $12\to16$ & $27\to22$ & $0.303\to0.372$ & $33.3/22.7$ & 0.5 \\
512 & $38\to45$ & $17\to14$ & $0.482\to0.585$ & $18.4/21.4$ & 0.3 \\
768 & $61\to70$ & $8.8\to7.5$ & $0.931\to1.092$ & $14.8/17.3$ & 0.2 \\
\bottomrule
\end{tabular}
\end{table}

At width 512, peak allocated memory rises from 38 to 45 GB and throughput falls from 17K to 14K tokens/s. With 8192 tokens per step, these throughputs imply step times of approximately 0.482 and 0.585 seconds: an 18.4\% memory increase and 21.4\% time increase. The same calculation gives time increases of 22.7\% at width 128 and 17.3\% at width 768. All runtime comparisons below use step times derived from these same reported throughputs.

The difference between arithmetic and runtime follows the operations being added. Updating $D$ requires elementwise addition or subtraction. Reusing it downstream adds normalization, scores, and weighted source sums. At fixed block count these operations are linear in hidden width, whereas the main backbone projections grow quadratically when MLP width scales with hidden width. Their fraction of analytical FLOPs therefore falls as the model widens.

The runtime implementation must still materialize and repeatedly read the additional sources. Later sublayers access all completed details, not just the current accumulator. Their storage and traffic remain relevant even when they contribute few multiply-add operations. The measured relative penalties decrease with width, but remain substantial at width 768: peak memory increases by 14.8\% and step time by 17.3\%. These are end-to-end measurements; they do not resolve the cost of individual kernels.

For a fixed training-token budget, these overheads are paid directly. HAARES is consequently a quality--cost trade rather than a free representation change. The practical question is whether the quality improvement can reduce the number of steps required for a chosen target enough to offset the slower steps. The recorded checkpoint crossings allow one budget-relative comparison of that question.

\subsection{Reaching the baseline's budget-best loss}

For each 201M seed, we take the best Block AttnRes validation loss within 30K steps as the target. HAARES first meets those targets on the 2K-step evaluation grid at 24K, 20K, and 14K steps (Table~\ref{tab:timetotarget}). The corresponding losses are 1.7961, 1.7835, and 1.8131. Thus every HAARES run reaches its same-seed baseline target before the full training budget is exhausted.

\begin{table}[!htbp]
\centering
\caption{HAARES time to reach the same-seed baseline budget-best loss. The Block time is the full 30K-step budget, not the baseline's measured first-attainment time. HAARES steps are first observed crossings on the 2K evaluation grid. Times are estimates derived from rounded throughput.}
\label{tab:timetotarget}
\small
\begin{tabular}{@{}lccccc@{}}
\toprule
Seed & Target loss & First \method{} step & \method{} loss & Time (Block$\to$\method{}) & Time saved \\
\midrule
42 & 1.7995 & 24K & 1.7961 & 4.02$\to$3.90 h & 2.9\% \\
123 & 1.7963 & 20K & 1.7835 & 4.02$\to$3.25 h & 19.0\% \\
2026 & 1.8213 & 14K & 1.8131 & 4.02$\to$2.28 h & 43.3\% \\
Mean & 1.8057 & 19.3K & 1.7976 & 4.02$\to$3.14 h & 21.7\% \\
\bottomrule
\end{tabular}
\end{table}

Converting these step counts using 14K tokens/s gives estimated HAARES times of 3.90, 3.25, and 2.28 hours, with mean 3.14 hours. A full 30K-step Block AttnRes run at 17K tokens/s takes an estimated 4.02 hours. Relative to that complete baseline budget, the mean saving is 21.7\%. The denominator is the full baseline budget, not its earliest target-reaching checkpoint: the available records do not include the latter. The targets are selected retrospectively from validation, and the times assume the reported average throughput applies over the compared intervals.

The step-time penalty gives a simple break-even condition. If $K_H$ and $K_B$ are compared step counts and $v_H,v_B$ their token throughputs, both runs process 8192 tokens per step, so
\begin{equation}
 \frac{8192K_H}{v_H}<\frac{8192K_B}{v_B}
 \quad\Longleftrightarrow\quad
 \frac{K_H}{K_B}<\frac{v_H}{v_B}.
 \label{eq:breakeven}
\end{equation}
Here $v_H/v_B=14/17\approx0.824$. Against a 30K-step baseline budget, HAARES must meet the target before about 24.7K steps to finish sooner. All three recorded crossings meet this threshold. Seed 42 does so narrowly at 24K, which explains its 2.9\% time saving; the 20K and 14K crossings leave wider margins.

This calculation connects the source overhead to a concrete training decision. A 21.4\% increase in time per step can be offset by a sufficiently large reduction in steps, but the balance is seed-dependent in the recorded runs. For applications that require an entire fixed-token schedule, the measured slowdown remains. For a validation-quality stopping target similar to those in the table, the existing crossings show how earlier attainment can compensate for that slowdown. Peak memory remains an independent constraint in either case.

The cost results concern one GPU, fp32, and an unfused source-handling implementation, including startup in the throughput window. Fusion or lower precision may alter the trade-off, but those configurations are outside the measurements reported here. The quality comparisons should be read together with these measured memory and throughput costs.

\section{Discussion}

\subsection{What the added source changes}

HAARES makes a block's early--late allocation available to later sublayers without retaining every output as an individual source. The raw inverse identifies the retained information exactly. The routing analysis then identifies its operational role: a weighted detail changes the relative coefficients of the early and late totals, while its slot also competes with cumulative sources for softmax mass. Both effects belong to the implemented method.

The controls support this distinction. Duplicate cumulative sources improve the parent, showing that a new slot can be useful without a new raw direction. The half-split detail improves further in all three ablation seeds. The tested random-sign detail does not reproduce the gain, and eliminating RMS matching retains most of it while worsening each seed relative to the full model. Neither the extra slot nor scale correction alone reproduces the loss reached by the structured source with learned allocation.

The alternative of doubling the cumulative block count clarifies the importance of coordinates. For fixed realized outputs, two half totals and their raw sum--difference pair contain the same information. Yet a router with normalized keys, nonnegative weights, and a type bias is sensitive to which coordinates it receives. HAARES also changes the active-prefix timing and keeps details out of the final readout. The four-block result depends on this representation and read interface; the number of stored vectors alone does not specify the computation.

Ablations trained from initialization compare complete models. They do not localize a learned function to a particular token, layer, or semantic operation. The architectural explanation here remains at the level tested: a structured early--late source improves the matched read interface in the strongest reported regime. Assigning that source a more specific learned role would require corresponding functional measurements.

\subsection{Choosing granularity under the reported training recipe}

The four-block, 48-layer configuration gives the most consistent gain across the main corpus, additional corpora, and paired model initializations. It combines long residual blocks with a small source bank. At six and eight blocks, the loss increase in one seed outweighs the reductions in the other seeds. The four-block half-split configuration is the best choice in this grid; finer partitions do not benefit from the same source addition.

The shared training recipe makes the parent and source-augmented comparisons directly interpretable: their token budgets, schedules, data order, and validation selection rule agree. It also fixes the scope of the result. The recipe uses fp32, a constant learning rate, and one data-order seed; independent optimizer tuning may change the relative behavior of residual parameterizations. The largest model has 453M parameters, with two seeds, and all models use context length 512. Evidence at these settings should be extended by measurements before being applied to substantially larger models or longer contexts.

Validation selection is part of the experimental question. Each endpoint is the best scheduled validation loss within the budget, and the architectural choices are compared on the same validation criterion. These endpoints measure attained validation quality, not performance on an independent test split. The domain table adds breadth but reports only seed means. The character and BPE experiments use different token units and split procedures, so the within-pipeline improvements can be compared in direction, but their loss magnitudes use different token units.

The experimental record leaves several preprocessing details unspecified. In particular, the BPE tokenizer is trained on an initial raw-byte segment before the random split, and disjointness from validation is not documented. The split unit and seed are also unspecified. On Yelp, the 50K-example evaluation split is used for checkpoint selection. The data and reproducibility appendices distinguish the recorded preprocessing steps from these unresolved details.

\subsection{Quality, memory, and implementation choices}

The source-bank change is small in parameters but appreciable in activations. At medium width, the observed 18.4\% peak-memory increase accompanies a 21.4\% increase in step time. Earlier attainment of the baseline's budget-best losses offsets the time penalty in the three recorded 201M comparisons, relative to running the baseline for its full budget. It does not remove the memory requirement. Whether the earlier target attainment is useful depends on both the training objective and the device's memory capacity.

The current implementation explicitly handles and reuses the additional sources. An optimized implementation could target source normalization, materialization, and weighted mixing, but the measured totals are those of the present unfused fp32 path. We do not substitute an arithmetic estimate for that runtime cost. Similarly, a finer or learned split would change both the information retained and the work required to expose it. The fixed midpoint offers a simple, reproducible source definition whose benefits and costs can be measured without a separate boundary-selection mechanism.

Public text and code corpora also require attention to their source licenses and inherited biases. This study evaluates language-model training loss, not privacy or harmful-generation behavior. The tables give the numerical comparisons and their seed-level variation where available. Reproducing the training runs also requires the original implementation, checkpoints, and raw logs. Appendix~\ref{app:repro} records the remaining implementation details separately from the mathematical definition.

\Needspace{23\baselineskip}
\section{Conclusion}

HAARES augments a block residual router with a single early--late detail alongside each cumulative source. The raw sum--difference pair recovers both half-block totals, and bounded magnitude matching supplies the detail to intermediate attention and MLP reads without changing the cumulative-only final readout. The method retains a coarse part of the block trajectory through a small change to the parent's source interface.

The 48-layer, four-block experiments show consistent paired improvements across the reported model widths and tokenization pipelines, with the same direction on four additional training corpora. Duplicate-source, random-sign, scale, and bias controls favor the full construction, while the shallow and finer-partition results identify configurations where its advantage is weaker or absent. Retaining a useful coordinate and giving it an appropriate routing treatment matter together.

The measured implementation pays for that quality with more activation memory and slower steps. Earlier attainment of the baseline's budget-best validation losses compensates for the time penalty relative to full-budget baseline runs in the recorded 201M comparisons. For deep models with long residual blocks, HAARES makes an early--late contrast available to intermediate sublayers while preserving the cumulative update. The resulting quality improvement must be weighed against the measured memory and runtime costs.

\section*{Acknowledgements}
The authors thank Professor Guoli Zhou of the School of Mathematics and Statistics, Chongqing University, for comments and suggestions during the development of this work.

\clearpage
\section*{Statements and Declarations}

\paragraph{Funding.}
This work was supported by the project ``Development of an Artificial Intelligence Large Model for Urban Transportation Vehicle Exterior Styling and Industrial Design'' (Project No. 2024CCA572), the Guangzhou Science and Technology Program (No. 2024A03J0416), and the Guangzhou Key Medical Discipline Construction Project (No. 2025-2027-11).

\paragraph{Author contributions.}
Kehan Wang, Shifeng Wu, and Zhenzhang Li jointly developed the research idea and methodology, carried out the experimental work and analysis, and drafted and revised the manuscript. Tingqiong Cui contributed to manuscript writing and provided financial support. Yibu Yang, Yanfei Hao, and Chenjie Liu contributed to discussions of the work and provided financial support.

\paragraph{Competing interests.}
Tingqiong Cui, Yibu Yang, Yanfei Hao, and Chenjie Liu are affiliated with the Technology R\&D Department of Chongqing CRRC Group Co., Ltd. Financial support is described above and in the author contributions. The authors declare no other competing financial or non-financial interests related to this work.

\paragraph{Data availability.}
The study uses the public datasets identified in Appendix~\ref{app:data}; no new dataset is introduced. Numerical results are reported in the main text and appendices. The observations underlying Figure~\ref{fig:pairedevidence} are also supplied in \texttt{paired\_results.csv}.

\paragraph{Code availability.}
The source files include the LaTeX manuscript, editable TikZ architecture diagram, and paired-effect plotting script. The mathematical definition and online algorithm are given in Section~\ref{sec:method}. The original model-training implementation, checkpoints, and raw logs are not included in this submission.

\paragraph{Ethics approval.}
The study uses public text and code corpora and does not report newly recruited human participants, animal experiments, or clinical interventions. No study-specific human or animal ethics approval is reported.

\paragraph{Use of generative AI in manuscript preparation.}
The authors did not use generative AI tools in the research or preparation of this manuscript.

\clearpage
\appendix
\section{Reproducibility specification}
\label{app:repro}

Tables~\ref{tab:reprobackbone} and~\ref{tab:reprotrain} consolidate the recorded implementation and training settings. Unspecified details are listed after the tables.

\begin{longtable}{@{}p{0.28\textwidth}p{0.67\textwidth}@{}}
\caption{Backbone, initialization, and routing settings.}\label{tab:reprobackbone}\\
\toprule Item & Recorded setting \\
\midrule
\endfirsthead
\caption[]{Backbone, initialization, and routing settings (continued).}\\
\toprule Item & Recorded setting \\
\midrule
\endhead
\bottomrule\endfoot
Architecture & Decoder-only causal PreNorm Transformer; attention followed by MLP in every layer \\
Attention & Eight heads; PyTorch \texttt{scaled\_dot\_product\_attention}; bias-free QKV and output projections; no attention dropout \\
Positions & RoPE, base 10,000; cached maximum length 2048 \\
Normalization & Affine RMSNorm, $\epsilon=10^{-6}$ \\
MLP & SwiGLU; family widths listed in Table~\ref{tab:modelconfig} \\
Other dropout & Zero in embedding, residual, and MLP paths \\
Weight tying & Input embedding and language-model head share weights \\
Weight initialization & Linear and embedding weights: $\mathcal N(0,0.02^2)$; ordinary biases and router queries: zero \\
Residual scales & ReZero gates: zero. LayerScale: $0.1$, $10^{-5}$, $10^{-6}$ at 12, 24, 48 layers \\
Detail bias & One scalar for each attention or MLP router, shared across visible detail slots, initialized to $-2$ \\
RMS matching & $\epsilon=10^{-6}$; clip multiplier to $[1/4,4]$; stop-gradient through multiplier \\
Readout sources & Embedding and completed cumulative sources only \\
\end{longtable}

\begin{longtable}{@{}p{0.28\textwidth}p{0.67\textwidth}@{}}
\caption{Training and evaluation settings.}\label{tab:reprotrain}\\
\toprule Item & Recorded setting \\
\midrule
\endfirsthead
\caption[]{Training and evaluation settings (continued).}\\
\toprule Item & Recorded setting \\
\midrule
\endhead
\bottomrule\endfoot
Optimizer & AdamW; $\beta=(0.9,0.95)$; optimizer $\epsilon=10^{-8}$ \\
Learning rate & $3\times10^{-4}$, constant; no warmup or decay \\
Weight decay & 0.01; exclusions described in Section~\ref{sec:methodology} \\
Gradient clipping & Global norm 1.0 \\
Batch and context & 16 sequences per GPU; context 512; 8192 tokens per step in single-GPU cost measurements \\
Precision & fp32 for accuracy and cost measurements \\
Model seeds & 42, 123, 2026; 453M comparison: 42 and 2026 \\
Data-order seed & 42 \\
Validation & One full pass every 2K steps; minimum within the training budget \\
Budgets & Small 12/24 layers: 50K steps. Small 48 layers and all medium/large configurations: 30K steps \\
Not used & Gradient accumulation, activation checkpointing, early stopping, \texttt{torch.compile}, dropout \\
Software & Ubuntu 22.04; Python 3.12; PyTorch 2.8.0; CUDA 12.8; cuDNN 9.x \\
\end{longtable}

Several details cannot be independently recovered from the available experiment record: the reduction axes and dtype used for the RMS-matching magnitude; handling of residual line fragments shorter than a complete chunk; exact attention-backend selection; the random-sign pattern and its construction seed; the unit and seed of the BPE split; and the per-checkpoint validation records. These are limits of the available experiment specification. The reported revision hashes identify data versions but do not verify these preprocessing choices.

\section{Dataset revisions and preprocessing}
\label{app:data}

Table~\ref{tab:datasetrevisions} lists the dataset hashes recorded from the experiment cache. The dataset providers and configurations are identified in Table~\ref{tab:data}; the corresponding dataset papers or repository are cited in Section~\ref{sec:methodology}.

\begin{table}[!htbp]
\centering
\caption{Recorded dataset revision identifiers.}
\label{tab:datasetrevisions}
\small
\begin{tabularx}{\textwidth}{@{}lX@{}}
\toprule
Dataset & Revision \\
\midrule
OpenWebText & \texttt{b4325f019c648b1641a1784748667e8b74e5e064} \\
WikiText-103 & \texttt{b08601e04326c79dfdd32d625aee71d232d685c3} \\
TinyStories & \texttt{f54c09fd23315a6f9c86f9dc80f725de7d8f9c64} \\
Yelp Review Full & \texttt{c1f9ee939b7d05667af864ee1cb066393154bf85} \\
CodeSearchNet & \texttt{bd0cf261e357a3eb5c8fba490d23ec1a1cd59555} \\
\bottomrule
\end{tabularx}
\end{table}

The character vocabulary assigns indices 0--3 to \texttt{<pad>}, \texttt{<unk>}, \texttt{<bos>}, and \texttt{<eos>}; the remaining entries are selected by training-corpus frequency. OpenWebText uses 30 of 80 parquet shards and assigns rows whose index is divisible by ten to validation, yielding an approximate 9:1 split. The available preprocessing specification does not state whether row indices reset at each shard.

The BPE pipeline trains on the first 100M bytes, tokenizes the complete recorded stream, and then applies a random 95/5 split. Its special tokens are \texttt{<pad>}, \texttt{<unk>}, \texttt{<s>}, and \texttt{</s>}. The recorded procedure does not establish document-level separation or disjointness between tokenizer-training bytes and validation. The reported approximate byte totals also refer to different stored stages; they should not be used as exact checksums of the character and BPE corpora.

The Yelp evaluation split contains 50K examples and is used for checkpoint selection in the reported experiments. No separate independent test evaluation is reported. Across datasets, the term validation refers to this operational role in model selection.

\FloatBarrier
\section{Reported perplexities and paired differences}
\label{app:crossppl}

Table~\ref{tab:crossppl} retains the original cross-domain loss/perplexity pairs. The available aggregate records do not distinguish averaging per-seed perplexities from exponentiating the mean loss. These rounded perplexities are retained for completeness; loss is used for all numerical comparisons.
\begin{table}[tbp]
\centering
\caption{Cross-domain reported mean best validation loss / perplexity. Only aggregate values are available for these domain experiments.}
\label{tab:crossppl}
\small
\begin{tabular}{@{}llccc@{}}
\toprule
Dataset & Method & 12L & 24L & 48L \\
\midrule
TinyStories & Standard Residual & $0.5555/1.74$ & $0.5391/1.71$ & $0.5219/1.69$ \\
& ReZero & $0.5815/1.79$ & $0.5571/1.75$ & $0.5520/1.74$ \\
& LayerScale & $0.5663/1.76$ & $0.5614/1.75$ & $0.5446/1.72$ \\
& Block AttnRes & $0.5549/1.74$ & $0.5310/1.70$ & $0.5355/1.71$ \\
& \method{} & $0.5569/1.75$ & $0.5313/1.70$ & $0.5167/1.68$ \\
\midrule
WikiText-103 & Standard Residual & $1.0337/2.81$ & $1.0014/2.72$ & $0.9783/2.66$ \\
& ReZero & $1.0952/2.99$ & $1.0533/2.87$ & $1.0312/2.80$ \\
& LayerScale & $1.0421/2.84$ & $1.0624/2.89$ & $1.0300/2.80$ \\
& Block AttnRes & $1.0325/2.81$ & $0.9990/2.72$ & $0.9742/2.65$ \\
& \method{} & $1.0258/2.79$ & $0.9754/2.65$ & $0.9607/2.61$ \\
\midrule
Yelp Review Full & Standard Residual & $0.9945/2.70$ & $0.9713/2.64$ & $0.9605/2.61$ \\
& ReZero & $1.0095/2.74$ & $0.9879/2.69$ & $0.9802/2.66$ \\
& LayerScale & $0.9977/2.71$ & $0.9874/2.68$ & $0.9710/2.64$ \\
& Block AttnRes & $0.9990/2.72$ & $0.9988/2.72$ & $0.9703/2.64$ \\
& \method{} & $1.0053/2.73$ & $0.9713/2.64$ & $0.9570/2.60$ \\
\midrule
CodeSearchNet-Python & Standard Residual & $0.9306/2.54$ & $0.9019/2.46$ & $0.8861/2.43$ \\
& ReZero & $0.9455/2.57$ & $0.9094/2.48$ & $0.8925/2.44$ \\
& LayerScale & $0.9277/2.53$ & $0.9139/2.49$ & $0.8863/2.43$ \\
& Block AttnRes & $0.9242/2.52$ & $0.8918/2.44$ & $0.8712/2.39$ \\
& \method{} & $0.9095/2.48$ & $0.8657/2.38$ & $0.8488/2.34$ \\
\bottomrule
\end{tabular}
\end{table}
\subsection{Exact paired loss differences}
\label{app:pairs}

Table~\ref{tab:pairdiffs} gives the values plotted in Figure~\ref{fig:pairedevidence}. Differences are computed within seed from the four-decimal reported losses. The 453M mean is 0.02325 before rounding and is shown as 0.0233, consistently in the text and table.
\begin{table}[tbp]
\centering
\caption{Paired reductions $\ell_{\mathrm{Block}}-\ell_{\method}$; positive favors HAARES. A dash indicates an unavailable run, not a zero effect.}
\label{tab:pairdiffs}
\small
\begin{tabular}{@{}lrrrr@{}}
\toprule
Setting & Seed 42 & Seed 123 & Seed 2026 & Mean \\
\midrule
201M character & 0.0168 & 0.0322 & 0.0449 & 0.0313 \\
453M character & 0.0383 & -- & 0.0082 & 0.0233 \\
204M BPE & 0.0086 & 0.0315 & 0.0049 & 0.0150 \\
\bottomrule
\end{tabular}
\end{table}
\subsection{Block-count and basis-ablation perplexities}\label{app:pairsppl}

Tables~\ref{tab:blockppl} and~\ref{tab:basispplextra} give the perplexities corresponding to the main-text ablation losses.
\begin{table}[tbp]
\centering
\caption{Reported perplexities for the block-count ablation. The corresponding loss values appear in Table~\ref{tab:blockablation}.}
\label{tab:blockppl}
\small
\begin{tabular}{@{}lccccc@{}}
\toprule
Method & $N$ & Seed 42 & Seed 123 & Seed 2026 & Mean \\
\midrule
Block AttnRes & 4 & $6.05$ & $6.03$ & $6.18$ & $6.09$ \\
\method{} & 4 & $5.95$ & $5.84$ & $5.91$ & $\mathbf{5.90}$ \\
Block AttnRes & 6 & $6.02$ & $5.91$ & $5.98$ & $\mathbf{5.97}$ \\
\method{} & 6 & $6.01$ & $5.88$ & $6.14$ & $6.01$ \\
Block AttnRes & 8 & $5.96$ & $6.02$ & $6.03$ & $\mathbf{6.00}$ \\
\method{} & 8 & $6.19$ & $6.01$ & $6.02$ & $6.07$ \\
\bottomrule
\end{tabular}
\end{table}

\begin{table}[tbp]
\centering
\caption{Perplexities for the source and stabilization ablations. Corresponding losses appear in Table~\ref{tab:basisablation}. A dagger identifies a value recalculated from its reported loss; unmarked entries retain the reported rounded values.}
\label{tab:basispplextra}
\small
\begin{tabular}{@{}lcccc@{}}
\toprule
Variant & Seed 42 & Seed 123 & Seed 2026 & Mean \\
\midrule
Block AttnRes & $6.05$ & $6.03$ & $6.18$ & $6.09$ \\
Duplicated $C\times2$ & $6.02$ & $5.91$ & $5.95$ & $5.96$ \\
Random-sign detail & $6.11$ & $6.05$ & $6.10$ & $6.09$ \\
Fixed detail bias $0$ & $6.07$ & $6.01$ & $6.10$ & $6.06$ \\
Fixed detail bias $-2$ & $6.15$ & $5.94^{\dagger}$ & $5.99^{\dagger}$ & $6.03$ \\
Fixed detail bias $-4$ & $6.19^{\dagger}$ & $6.15$ & $6.23$ & $6.19$ \\
\method{} without RMS match & $5.97$ & $5.90$ & $5.92$ & $5.93$ \\
\method{} & $\mathbf{5.95}$ & $\mathbf{5.84}$ & $\mathbf{5.91}$ & $\mathbf{5.90}$ \\
\bottomrule
\end{tabular}
\end{table}
\FloatBarrier
Three seed-level perplexities require arithmetic correction. For fixed detail bias $-2$, losses 1.7816 and 1.7895 imply perplexities 5.94 and 5.99, rather than the reported 5.93 and 5.98. For fixed bias $-4$, loss 1.8226 implies 6.19, rather than 6.18. These discrepancies exceed the tolerance attributable to rounding loss to four decimal places. The corrected values are marked in Table~\ref{tab:basispplextra}; the loss observations are unchanged. By contrast, an aggregate perplexity need not equal the exponential of an aggregate loss: averaging and exponentiation do not commute. The reported aggregate perplexities are retained, including the 48-layer Block mean of 6.09.

\section{Cost calculations}
\label{app:cost}

For throughput $v$ in tokens per second, the step time is $8192/v$, the time for $K$ steps is $8192K/(3600v)$ hours, and HAARES step-time overhead relative to Block AttnRes is $v_{\mathrm{Block}}/v_{\method}-1$. Applying these expressions to 17K and 14K tokens/s gives 4.0157 hours for 30K baseline steps and 3.9010, 3.2508, and 2.2756 hours for 24K, 20K, and 14K HAARES steps. Their mean is 3.1424 hours. Rounding only after calculation gives the values in Table~\ref{tab:timetotarget}.

The estimates use the full baseline budget and the average throughput values reported in Table~\ref{tab:cost}. The analytical FLOP increments count detail accumulation, RMS matching, and routing arithmetic; launch latency and memory traffic belong to the measured runtime rather than that arithmetic count. The throughput window starts at step zero and includes initial loading and setup.

\FloatBarrier
\bibliographystyle{plainnat}
\bibliography{references}
\end{document}